\documentclass[letterpaper, 10pt, conference]{ieeeconf}

\IEEEoverridecommandlockouts

\usepackage{siunitx}
\usepackage{relsize}
\usepackage{ifthen}
\usepackage[colorinlistoftodos]{todonotes}

\usepackage[caption=false]{subfig}

\usepackage{graphics} 
\usepackage{rotating}
\usepackage{color}
\usepackage{enumerate}
\usepackage[T1]{fontenc}
\usepackage{psfrag}
\usepackage{epsfig} 
\usepackage{booktabs}
\usepackage{graphicx,url}
\usepackage{multirow}
\usepackage{array}
\usepackage{latexsym}
\usepackage{amsfonts}
\usepackage{amsmath}
\usepackage{amssymb}

\usepackage{amsthm}
\usepackage{mathtools}
\usepackage{xstring}
\usepackage{algorithm} 
\usepackage[noEnd=true, indLines=true, commentColor=NavyBlue, endLComment=~, beginComment=//~]{algpseudocodex} 
\usepackage{xpatch}
\usepackage{multirow}
\usepackage{xcolor}
\usepackage[dvipsnames]{xcolor}
\usepackage{prettyref}
\usepackage{flexisym}
\usepackage{bigdelim}
\usepackage{listings}
\let\labelindent\relax
\usepackage{enumitem}
\usepackage{xspace}
\usepackage{bm}
\graphicspath{{./figures/}}
\usepackage{tikz}
\usetikzlibrary{matrix,calc}
\usepackage[acronym]{glossaries}
\usepackage[bookmarks=true]{hyperref}
\usepackage[capitalize]{cleveref}
\usepackage{subcaption}
\usepackage{aliascnt}

\usepackage{mdwlist}

\let\itemize\undefined
\makecompactlist{itemize}{stditemize}

\newrefformat{prob}{Problem\,\ref{#1}}
\newrefformat{def}{Definition\,\ref{#1}}
\newrefformat{sec}{Section\,\ref{#1}}
\newrefformat{sub}{Section\,\ref{#1}}
\newrefformat{prop}{Proposition\,\ref{#1}}
\newrefformat{app}{Appendix\,\ref{#1}}
\newrefformat{alg}{Algorithm\,\ref{#1}}
\newrefformat{cor}{Corollary\,\ref{#1}}
\newrefformat{thm}{Theorem\,\ref{#1}}
\newrefformat{lem}{Lemma\,\ref{#1}}
\newrefformat{fig}{Fig.\,\ref{#1}}
\newrefformat{tab}{Table\,\ref{#1}}

\theoremstyle{definition}

\theoremstyle{definition}
\newtheorem{problem}{Problem} 
\crefname{problem}{Problem}{Problems}

\crefname{trule}{Rule}{Rules}

\newtheorem{corollary}{Corollary}[section]
\crefname{corollary}{Corollary}{Corollaries}

\crefname{conjecture}{Conjecture}{Conjectures}

\newtheorem{lemma}{Lemma}[section]
\crefname{lemma}{Lemma}{Lemmas}

\crefname{assumption}{Assumption}{Assumptions}

\newtheorem{definition}{Definition}[section]
\crefname{definition}{Definition}{Definitions}

\crefname{proposition}{Proposition}{Propositions}

\newtheorem{remark}{Remark}[section]
\crefname{remark}{Remark}{Remarks}

\crefname{example}{Example}{Examples}

\newcommand{\bdmath}{\begin{dmath}}
\newcommand{\edmath}{\end{dmath}}
\newcommand{\beq}{\begin{equation}}
\newcommand{\eeq}{\end{equation}}
\newcommand{\bdm}{\begin{displaymath}}
\newcommand{\edm}{\end{displaymath}}
\newcommand{\bea}{\begin{eqnarray}}
\newcommand{\eea}{\end{eqnarray}}
\newcommand{\beal}{\beq \begin{array}{ll}}
\newcommand{\eeal}{\end{array} \eeq}
\newcommand{\beas}{\begin{eqnarray*}}
\newcommand{\eeas}{\end{eqnarray*}}
\newcommand{\ba}{\begin{array}}
\newcommand{\ea}{\end{array}}
\newcommand{\bit}{\begin{itemize}}
\newcommand{\eit}{\end{itemize}}
\newcommand{\ben}{\begin{enumerate}}
\newcommand{\een}{\end{enumerate}}

\newcommand{\calB}{{\cal B}}
\newcommand{\calC}{{\cal C}}

\newcommand{\calG}{{\cal G}}

\newcommand{\calI}{{\cal I}}

\newcommand{\calR}{{\cal R}}
\newcommand{\calS}{{\cal S}}
\newcommand{\calT}{{\cal T}}

\renewcommand{\boldsymbol}[1]{{\bm #1}}

\newcommand{\hide}[1]{}

\newcommand{\hiddenText}{{\color{gray} hidden text.}}
\newcommand{\hideWithText}[1]{\hiddenText}

\newcommand{\vtau}{\boldsymbol{\tau}}

\newcommand{\blue}[1]{{\color{blue}#1}}

\newcommand{\linkToPdf}[1]{\href{#1}{\blue{(pdf)}}}
\newcommand{\linkToPpt}[1]{\href{#1}{\blue{(ppt)}}}
\newcommand{\linkToCode}[1]{\href{#1}{\blue{(code)}}}
\newcommand{\linkToWeb}[1]{\href{#1}{\blue{(web)}}}
\newcommand{\linkToVideo}[1]{\href{#1}{\blue{(video)}}}
\newcommand{\linkToMedia}[1]{\href{#1}{\blue{(media)}}}
\newcommand{\award}[1]{\xspace} 

\newcommand{\myparagraph}[1]{\noindent\textbf{#1}}
\newcommand*\circled[1]{\tikz[baseline=(char.base)]{
            \node[shape=circle,draw,inner sep=0.5pt] (char) {#1};}}

\newcommand{\traj}{\tau}

\newcommand{\cB}{\mathcal{B}}

\newcommand{\cI}{\mathcal{I}}

\newcommand{\bmat}{\left[ \begin{array}}
\newcommand{\emat}{\end{array}\right]}

\newcommand{\tup}[1]{\left( #1\right)}

\newacronym{acr:mapf}{MAPF}{Multi-Agent Path Finding}
\newacronym{acr:fico}{FICO}{Finite-Horizon Closed-Loop Factorization}
\newacronym{acr:soc}{SOC}{Sum of Cost}
\newacronym{acr:mpc}{MPC}{Model Predictive Control}
\newacronym{acr:rl}{RL}{Reinforcement Learning}
\newacronym{acr:ert}{ERT}{Execution Response Time}
\newacronym{acr:accbs}{ACCBS}{Anytime Closed-Loop Conflict-Based Search}
\newacronym{acr:cbs}{CBS}{Conflict-Based Search}
\newacronym{acr:mapflh}{MAPF-LH}{Multi-Agent Path Finding with Logical Heterogeneity}
\newacronym{acr:daccbs}{CDCBS}{Certificate-Driven Conflict-Based Search}

\usepackage{ifthen}
\newboolean{anonymous}
\setboolean{anonymous}{false}
\usepackage{changes}
\usepackage{cite}
\usepackage[font=footnotesize]{caption}

\renewcommand{\baselinestretch}{0.98}

\begin{document}

\title{\LARGE \bf
Certificate-Driven Closed-Loop Multi-Agent Path Finding \\ with Inheritable Factorization
}

\ifthenelse{\boolean{anonymous}}{
  \author{Anonymous Authors}
}{
\author{Jiarui Li,
        Runyu Zhang,
        Gioele Zardini
\thanks{The authors are with the Laboratory for Information and Decision Systems, Massachusetts Institute of Technology, Cambridge, MA, USA (e-mails: \{jiarui01, runyuzha, gzardini\}@mit.edu).}
\thanks{This work was supported by Prof. Zardini's grant from the MIT Amazon Science Hub, hosted in the Schwarzman College of Computing. Runyu Zhang was supported by the MIT Postdoctoral Fellowship
Program for Engineering Excellence.}
}
}

\maketitle
\thispagestyle{empty}
\pagestyle{empty}

\begin{abstract}
Multi-agent coordination in automated warehouses and logistics is commonly modeled as the \gls{acr:mapf} problem.
Closed-loop \gls{acr:mapf} algorithms improve scalability by planning only the next movement and replanning online, but this finite-horizon viewpoint can be shortsighted and makes it difficult to preserve global guarantees and exploit compositional structure.
This issue is especially visible in \gls{acr:accbs}, which applies \gls{acr:cbs} over dynamically extended finite horizons but, under finite computational budgets, may terminate with short active prefixes in dense instances.

We introduce {certificate trajectories} and their associated {fleet budget} as a general mechanism for filtering closed-loop updates.
A certificate provides a conflict-free fallback plan and a monotone upper bound on the remaining cost; accepting only certificate-improving updates yields completeness. The same budget information induces a {budget-limited factorization} that enables global, inheritable decomposition across timesteps.
Instantiating the framework on \gls{acr:accbs} yields \gls{acr:daccbs}.
Experiments on benchmark maps show that \gls{acr:daccbs} achieves more consistent solution quality than \gls{acr:accbs}, particularly in dense settings, while the proposed factorization reduces effective group size.
\end{abstract}

\section{Introduction}
Multi-agent coordination is central to large-scale automated warehouses and fulfillment centers, which increasingly underpin global supply chains, e-commerce, and manufacturing~\cite{d2012guest}.
\gls{acr:mapf}, in which agents move from prescribed starts to goals on a shared graph while avoiding collisions, is a canonical abstraction of this coordination problem~\cite{stern2019mapf}.
As a result, \gls{acr:mapf} has become a core problem at the intersection of robotics and AI~\cite{wurman2008coordinating,standley2010finding,shaoul2024accelerating,wang2025lns2+,paul2022multi}.

A central challenge in \gls{acr:mapf} is the trade-off between scalability and guarantees.
Solving \gls{acr:mapf} optimally, or even with strong performance guarantees, is NP-hard~\cite{yu2013planning}.
Existing approaches therefore split broadly into two classes.
One line of work emphasizes scalability, often at the expense of guarantees, using heuristic or approximate methods such as CA*~\cite{silver2005cooperative}, OD~\cite{standley2010finding}, MAPF-LNS2~\cite{li2022mapf}, LaCAM~\cite{okumura2023lacam}, and their anytime variants~\cite{li2021anytime,okumura2023improving,okumura2023engineering,gandotra2025anytime}.
Another line of work provides strong optimality or bounded-suboptimality guarantees, including \gls{acr:cbs}~\cite{sharon2015conflict} and its variants~\cite{boyarski2015icbs,barer2014suboptimal,li2021eecbs}, as well as BCP~\cite{lam2022branch}, ICTS~\cite{sharon2013increasing}, and M*~\cite{wagner2015subdimensional}.
These guaranteed methods, however, often scale poorly because they must resolve conflicts over full joint trajectories: new conflicts can arise while old ones are being resolved, and conflicts that occur far in the future still need to be handled explicitly.

Closed-loop \gls{acr:mapf} offers a different compromise.
Instead of planning an entire joint trajectory at once, the controller outputs only the next movement and replans online~\cite{li2025fico,li2026adaptive,li2021lifelong}.
This relaxes the need to optimize over the full horizon and naturally motivates finite-horizon designs, which defer later conflicts and can make certifiable planning methods such as \gls{acr:cbs} applicable to larger fleets.
In particular, our recent \gls{acr:accbs} algorithm~\cite{li2026adaptive} applies \gls{acr:cbs} over an active prefix and progressively extends the planning horizon within the available per-step computational budget.

However, finite-horizon reasoning introduces a new difficulty.
In dense instances, where conflict resolution is intrinsically harder, \gls{acr:accbs} often times out at short horizons under realistic computation budgets.
The resulting shortsightedness can degrade solution quality and make the closed-loop behavior unstable.
At the same time, because future interactions beyond the current horizon remain weakly constrained, the underlying compositionality of the fleet is difficult to exploit in a way that persists across timesteps.
Existing groupings in finite-horizon closed-loop planning are therefore often local and must be recomputed repeatedly.

This paper addresses both issues with a certificate-based framework for closed-loop \gls{acr:mapf}.
At each timestep, we maintain a set of \emph{certificate trajectories} from the current state to the goals, together with their cost, called the \emph{fleet budget}.
The certificate serves as a feasible incumbent plan: candidate closed-loop updates are accepted only if they improve the current certificate.
This restores a global notion of progress to finite-horizon closed-loop planning.
Moreover, the fleet budget constrains how far future motion can deviate from individual shortest paths, which in turn enables a \emph{budget-limited factorization} of the fleet.
The resulting factorization is global and inheritable across timesteps.

\subsection{Statement of Contribution}
This paper makes three contributions.
First, we introduce certificate trajectories and the fleet budget as a general mechanism for filtering closed-loop \gls{acr:mapf} updates.
The certificate provides a conflict-free fallback plan at every timestep, and the monotone decrease of the fleet budget yields a concise completeness argument.
Second, we develop \emph{budget-limited factorization}.
By using the fleet budget to define budget-limited reachable regions, we derive a factorization whose independence relations are global and inheritable across timesteps, thereby enabling closed-loop planning to exploit compositional structure more effectively.
Third, we instantiate the framework on \gls{acr:accbs}, yielding \gls{acr:daccbs}.
In \gls{acr:daccbs}, conflict-free active prefixes found by the finite-horizon \gls{acr:cbs} search are completed with a conflict-free backup rollout to form certificate candidates, and the incumbent certificate is updated only when the candidate has lower cost.
Experiments on benchmark maps show that \gls{acr:daccbs} is substantially more robust than \gls{acr:accbs} in dense settings, while the proposed factorization can strongly reduce effective group size.

Although this paper focuses on \gls{acr:accbs}, the certificate-based framework itself is more general and may be applicable to other closed-loop \gls{acr:mapf} algorithms as well.

\section{Problem Definition and Preliminaries} \label{sec:preliminary}

\subsection{\gls{acr:mapf}}
We first introduce the notation used throughout the paper.
An \gls{acr:mapf} problem is specified by an \emph{\gls{acr:mapf} instance}.

\begin{definition}[\gls{acr:mapf} instance]
\label{def:mapf-inst}
An \gls{acr:mapf} \emph{instance} is a tuple~$\cI=\tup{G,A,\rho_s, \rho_g}$ where~$G=\tup{V,E}$ is a directed reflexive graph,~$A=\{a_1, \ldots, a_N\}$ is a set of~$N$ agents, and~$\rho_s\colon A\to V$ and~$\rho_g\colon A\to V$ map each agent~$a\in A$ to its start vertex~$s_a=\rho_s(a)\in V$ and goal vertex~$g_a=\rho_g(a)\in V$, respectively.
\end{definition}

\begin{definition}[Trajectories and conflicts]
Let~$a_i\in A$ be an agent.
A \emph{trajectory} of~$a_i$, denoted~$\traj^{a_i}=[v^{a_i}_0,v^{a_i}_1,\ldots,v^{a_i}_{T_i}]$, is a finite sequence of vertices such that~$\tup{v^{a_i}_t,v^{a_i}_{t+1}}\in E$ for all~$t\in\{0,\ldots,T_i-1\}$.
Given two trajectories~$\traj^{a_i}=[v^{a_i}_0,\ldots,v^{a_i}_T]$ and~$\traj^{a_j}=[v^{a_j}_0,\ldots,v^{a_j}_T]$ of equal length $T$, we define:
a) \emph{Vertex Conflict}:~$\exists t\in \{0,\ldots,T\}$ such that~$v^{a_i}_t = v^{a_j}_t$;
b) \emph{Edge Conflict}:~$\exists t\in \{0,\ldots,T\}$ such that~$v^{a_i}_t = v^{a_j}_{t-1}$ and $v^{a_i}_{t-1} = v^{a_j}_t$.
Two trajectories are \emph{conflict-free} if they exhibit neither vertex conflicts nor edge conflicts.
\end{definition}

We use $\vtau^A$ to denote the collection of the trajectories for all agents in $A$\footnote{
When referring to a joint collection of trajectories, we implicitly pad shorter trajectories with wait actions at their terminal vertices so that all trajectories share a common makespan (possible because $G$ is reflexive).
}.
Our recent work~\cite{li2025fico} introduced the \emph{\gls{acr:mapf} system}, a system-level framework that integrates planning and execution within a single feedback loop and unifies different \gls{acr:mapf} variants.
Within this framework, the \gls{acr:mapf} problem can be cast as a controller-synthesis problem.

\begin{definition}[\gls{acr:mapf} system and the unified \gls{acr:mapf} problem] \label{def:mapf-system}
Let~$\calI = \tup{G, A, \rho_s, \rho_g}$ be the \gls{acr:mapf} instance at time $t$.
A \emph{state} is a map~$x_t\colon A\to V$, mapping each agent to its current vertex.
For~$a\in A$, we write~$x_t(a)$ for the position of~$a$ at time~$t$.
A (per-agent) \emph{movement} command is a map~$u_t\colon A\to E$, with the consistency constraint
\begin{equation*}
\forall a\in A:\; u_t(a) = \tup{x_t(a),\, x_{t+1}(a)} \in E,
\end{equation*}
for the executed next state~$x_{t+1}$.
The \emph{\gls{acr:mapf} system} models planning and execution in a feedback loop.
The \emph{unified \gls{acr:mapf} problem} is to design a controller~$g \colon \tup{x_t,\cI,t} \mapsto \hat{u}_t$
that outputs a planned movement $\hat{u}_t$ at each time $t$ until a terminal condition is met.
We call $g$ \emph{open-loop} if $\hat{u}_t = g(\cI,t)$, i.e., the entire trajectories are computed a priori, and \emph{closed-loop} if $\hat{u}_t = g(x_t,\cI,t)$, i.e., movements are recomputed online from the current state.
\end{definition}

In this paper, we focus on the \emph{one-shot} \gls{acr:mapf} problem, defined as follows.

\begin{problem}[One-shot \gls{acr:mapf} problem] \label{prob:one-shot-mapf}
    The \emph{one-shot} \gls{acr:mapf} problem is a unified \gls{acr:mapf} problem where all movements are perfectly executed, meaning $\hat u_t = u_t$ always holds, and the instance $\calI$ remains unchanged across timesteps. The system evolution terminates when all agents reach their goals.
\end{problem}

\subsection{\gls{acr:cbs}, finite-horizon \gls{acr:cbs}, and \gls{acr:accbs}} \label{sec:preliminary-cbs}
Among \gls{acr:mapf} algorithms with formal performance guarantees, \acrfull{acr:cbs}~\cite{sharon2015conflict} is a foundational approach.
\gls{acr:cbs} performs a best-first search over a binary \emph{constraint tree}, where each node stores a set of constraints, a set of per-agent trajectories satisfying those constraints, and the corresponding joint cost.
At each iteration, the algorithm expands the lowest-cost node, detects a conflict, and creates two child nodes by adding one additional constraint to one of the involved agents in each child and replanning accordingly.
Expanding the first conflict-free node yields an optimal joint trajectory.
However, like other \gls{acr:mapf} algorithms with guarantees, \gls{acr:cbs} and its variants, including ICBS~\cite{boyarski2015icbs}, ECBS~\cite{barer2014suboptimal}, and EECBS~\cite{li2021eecbs}, struggle to scale because of their inherently \emph{open-loop} nature: conflicts must be resolved over entire trajectories rather than only over near-term actions.
Finite-horizon \gls{acr:cbs} alleviates this burden by restricting conflict resolution to a fixed planning window.
While this can substantially reduce computation, it introduces the horizon length as an additional design parameter and, in general, sacrifices the optimality and completeness guarantees of full-horizon \gls{acr:cbs}~\cite{bachtiar2016continuity}.

Recently, we proposed \gls{acr:accbs}, which progressively enlarges the search horizon of finite-horizon \gls{acr:cbs}~\cite{li2026adaptive}.
\gls{acr:accbs} utilizes the \emph{active prefix}, whose length is denoted by the \emph{running horizon}\footnote{Formally, the length is equal to the running horizon minus one. }, to represent the finite-horizon trajectory segments over which \gls{acr:cbs} is invoked to resolve conflicts and optimize.
For computational efficiency, the low-level individual trajectory planner generates trajectories over a nominal horizon $H_\max \gg h_r$ for all agents in $A$, and the set of all trajectories is denoted as $\vtau^{A,H_{\max}}$.

\begin{definition}[Running horizon and active prefix]
The \emph{active prefix} of the running horizon~$h_r$ associated with~$\vtau^{A,H_{\max}}$ is the joint trajectory~$\textstyle{\vtau^{A,h_r \mid H_{\max}}_{\mathrm{finite}}
        = \{\tau^{a_1,h_r \mid H_{\max}}_{\mathrm{finite}},\ldots,
            \tau^{a_N,h_r \mid H_{\max}}_{\mathrm{finite}}\}}$,
    where for each agent~$a_i$,~$\tau^{a_i,h_r \mid H_{\max}}_{\mathrm{finite}}
        = [v^{a_i}_0,\ldots,v^{a_i}_{h_r}].$
\end{definition}

Mirroring the classical \gls{acr:cbs}, \gls{acr:accbs} constructs and searches over a constraint tree, where the node is defined as follows.

\begin{definition}[Node in the constraint tree]
\label{def:node}
Fix a horizon~$H$ and a time~$t$.
A \emph{node} in the finite-horizon constraint tree is a triple~$n = \tup{C(n),\, \vtau^{A^t,H_\max}_{\mathrm{finite}}(n),\, J_{h_r}(n)}$, where
$C(n)$ is a finite constraint set (vertex and edge constraints),
$\vtau^{A^t,H_\max}_{\mathrm{finite}}(n)
        = \{\tau^{a_1,H_\max}_{\mathrm{finite}},\ldots,
            \tau^{a_N,H_\max}_{\mathrm{finite}}\}$ is a set of~$H_\max$-step trajectories at time~$t$ for all agents~$a_i\in A$ that satisfies~$C(n)$, and~$J_{h_r}(n) = \mathcal{C}\bigl(\vtau^{A,h_r\mid H_\max}_{\mathrm{finite}}(n)\bigr)$ is the \emph{node cost}, which reflects the cost of the active prefix w.r.t. the running horizon~$h_r$.
\end{definition}

\begin{definition}[Vertex constraints, edge constraints, and satisfaction]
A \emph{vertex constraint} is a tuple $c^{\mathrm{vtx}}=\tup{a_i,t_i,v_i}$, where~$a_i\in A$, $v_i\in V$, and~$t_i\in\{0,\dots,H_\max\}$, prohibiting agent $a_i$ from being at vertex $v_i$ at time $t_i$.
An \emph{edge constraint} is a tuple $c^{\mathrm{edg}}=\tup{a_i,t_i,e_i}$, where $a_i\in A$, $e_i=\tup{u_i,w_i}\in E$, and $t_i\in\{0,\dots,H_\max-1\}$, prohibiting agent $a_i$ from traversing edge $e_i$ from time $t_i$ to $t_i+1$.
A \emph{constraint set} is a finite set $C=\{c^{\mathrm{vtx}}_1,\ldots,c^{\mathrm{vtx}}_p,c^{\mathrm{edg}}_1,\ldots,c^{\mathrm{edg}}_q\}$.

The active prefix with running horizon $h_r$~$\vtau_{\mathrm{finite}}^{A^t,h_r\mid H_\max}$
is a set of $h_r$-step trajectories at time~$t$.
We say that~$\vtau_{\mathrm{finite}}^{A^t,h_r\mid H_\max}$ \emph{satisfies} a
constraint set~$C$ if every constraint~$c\in C$ is obeyed by the corresponding agent trajectory: for~$c=\tup{a_i,t_i,v_i}$ (vertex),~$v^{a_i}_{t_i}\neq v_i$; for~$c=\tup{a_i,t_i,\tup{u_i,w_i}}$ (edge),~$\tup{v^{a_i}_{t_i}, v^{a_i}_{t_i+1}} \neq \tup{u_i,w_i}$.
\end{definition}

The construction rule of the constraint tree in~\gls{acr:accbs}, as introduced in the later parts, ensures that no constraint outside the active prefix can be added into the nodes.
\cref{def:node} requires that the nodes be sorted by cost, which is based on the cost of the active prefix.
In general, the cost for finite-horizon trajectories is formulated as follows.

\begin{definition}[Cost function]
    \label{def:cost-function}
Let the active prefix~$\tau_{\mathrm{finite}}^{a_i,hr \mid H_\max}
    = [v^{a_i}_0, \ldots, v^{a_i}_{h_r}]$ be an~$h_r$-step trajectory of agent~$a_i$ at time~$t$.

    \noindent \textbf{Running cost.}
    For each time step~$\ell\in {0,\ldots, h_r-1}$, define
    \begin{equation*}
        p^{a_i}(v^{a_i}_\ell)
        =
        \begin{cases}
            1, & v^{a_i}_\ell \neq \rho_g(a_i), \\
            0, & \text{otherwise}.
        \end{cases}
    \end{equation*}
    Thus, the running cost penalizes each time step in which~$a_i$ has not yet reached its goal.

    \noindent \textbf{Terminal cost.}
    Let~$\gamma(v,a_i)$ be the length of a shortest path in~$G$ from~$v \in V$ to the goal~$\rho_g(a_i)$.
    The terminal cost is~$q^{a_i}(v^{a_i}_{h_r}) = \gamma(v^{a_i}_{h_r}, a_i)$, which approximates the remaining cost-to-go beyond the~$h_r$-step horizon.

    The total cost of a single trajectory is the sum of the running and terminal costs, and the cost of a set of trajectories is the sum over agents:
    \begin{align*}
        \mathcal{C}({\traj^{a_i,h_r \mid H_\max}_{\mathrm{finite}}}) &= \textstyle \sum_{t=0}^{h_r-1} p^{a_i}(v^{a_i}_t) + q^{a_i}(v^{a_i}_{h_r}) \\
        \mathcal{C}({\vtau^{A^t,h_r\mid H_\max}_{\mathrm{finite}}}) &= \textstyle \sum_{a_i} \mathcal{C}(\traj^{a_i,h_r\mid H_\max}_{\mathrm{finite}})
    \end{align*}
\end{definition}

At each timestep, $h_r$ is initialized as 1, and the \gls{acr:accbs} constructs a constraint tree and searches for the node whose active prefix is conflict-free. Upon finding such a node, \gls{acr:accbs} increments the running horizon and continues the search on the same tree thanks to the cost invariance~\cite[Lemma~III.1]{li2026adaptive} and the resulting tree reusability~\cite[Proposition~III.1]{li2026adaptive}. Ultimately, the movement output is extracted from the longest conflict-free active prefix found.

\myparagraph{Limitations of ACCBS.}
Even with finite-horizon computation, the construction and search on the constraint tree are still intractably time-consuming, especially in dense environments.
As a result, \gls{acr:accbs} may time out prematurely and generate output movement from a short active prefix, exhibiting shortsightedness, where the substantial finite-horizon approximation error can lead to low-quality decision-making and unstable solution performance.
Moreover, \gls{acr:accbs} fails to exploit the internal compositionality of the \gls{acr:mapf} problem by performing factorization when planning for single timesteps (like FICO~\cite{li2025fico} does) due to the fact that the extending running horizon can infinitely expand the possible reachable region of each agent, which leads to all agents being in the same group.

\section{Certificates and Valid Updates} \label{sec:certificates}
Under finite computational budgets, \gls{acr:accbs} may terminate with a short active prefix and execute a movement that is only weakly informed by future agent interactions.
We address this limitation by maintaining, at every timestep, a conflict-free full-horizon plan called a \emph{certificate}.
The certificate plays two roles.
First, it provides a feasible fallback plan whenever the finite-horizon search fails to find an improvement.
Second, it acts as a filter on closed-loop updates: a newly generated plan is accepted only if it improves the incumbent certificate.
This turns the closed-loop decision problem from ``find the next movement'' into ``find a better certificate,'' while preserving a valid plan at all times.

\subsection{Certificate trajectories, fleet budget, and backup controller}

\begin{definition}[Certificate trajectories and fleet budget] \label{def:certificate-budget}
For a one-shot \gls{acr:mapf} instance, a \emph{certificate} at time $t$ is a pair $(\vtau_{\mathrm{cert}}^{A,t}(x_t),\mathcal{B}_t)$, where $\vtau_{\mathrm{cert}}^{A,t}(x_t)=\{\traj_{\mathrm{cert}}^{a_1,t},\ldots,\traj_{\mathrm{cert}}^{a_N,t}\}$ is a set of \emph{certificate trajectories} and~$\mathcal{B}_t$ is its associated \emph{fleet budget}.
Each trajectory
\[
\traj_{\mathrm{cert}}^{a_i,t}=[v_{\mathrm{cert}}^{a_i,t\mid t},\ldots,v_{\mathrm{cert}}^{a_i,M_t\mid t}]
\]
satisfies:
\begin{itemize}
    \item[(i)] the trajectories in~$\vtau_{\mathrm{cert}}^{A,t}(x_t)$ are mutually conflict-free;
    \item[(ii)] $v_{\mathrm{cert}}^{a_i,t\mid t}=x_t(a_i)$ and $v_{\mathrm{cert}}^{a_i,M_t\mid t}=\rho_g(a_i)$ (i.e., each trajectory originates at the agent's current state and terminate at its respective goal); and
    \item[(iii)] $\mathcal{C}(\vtau_{\mathrm{cert}}^{A,t}(x_t))=\mathcal{B}_t$ (i.e., the total cost of the trajectories equals the fleet budget).
\end{itemize}
\end{definition}

\begin{remark}[Interpretation of certificates]
\label{rem:interpretation-certificates}
A certificate is both a feasible fallback plan and an upper bound on the remaining solution cost from the current state.
In contrast to a node in a \gls{acr:cbs} constraint tree, which lower-bounds the cost while future conflicts may remain unresolved, a certificate resolves all conflicts and therefore certifies a feasible closed-loop continuation.
\end{remark}

At the first timestep, the certificates are initialized by a rollout of a backup planner.

\begin{definition}[Backup controller] \label{def:bk-controller}
A \emph{backup controller}~$\pi_{\mathrm{bk}}$ is a complete\footnote{Completeness ensures a finite fleet budget.}{} \gls{acr:mapf} controller that, given a current state restricted to a set of agents~$\bar A\subseteq A$, returns a conflict-free trajectory set from the current positions of agents in~$\bar A$ to their respective goals.
\end{definition}

\begin{remark}[Possible backup controllers]
LaCAM~\cite{okumura2023lacam} and its variants~\cite{okumura2023improving,okumura2023engineering} are natural choices because they combine completeness with strong runtime performance.
Higher-quality but typically more expensive alternatives, such as EECBS~\cite{li2021eecbs} and M*~\cite{wagner2015subdimensional}, can also be used when tighter certificates justify the additional computation.
\end{remark}

\subsection{Certificate updates and completeness} \label{sec:updating-certificates}
The certificate is maintained and improved as the \gls{acr:mapf} system evolves.
The key object is not the raw output of the closed-loop planner, but the current certificate that survives after filtering.

\begin{definition}[Valid certificate update]
\label{def:valid-update}
A closed-loop \gls{acr:mapf} algorithm performs a \emph{valid certificate update} at time~$t$ if it maintains a certificate~$(\vtau_{\mathrm{cert}}^{A,t}(x_t),\mathcal{B}_t)$ satisfying
\begin{equation}
\label{eq:budget-update-across-timesteps}
\textstyle \mathcal{B}_t \le \mathcal{B}_{t-1} - \sum_{a_i\in A} p^{a_i}(x_{t-1}(a_i)).
\end{equation}
The executed movement at time $t$ is the first-step movement extracted from~$\vtau_{\mathrm{cert}}^{A,t}(x_t)$.
\end{definition}

Specifically, a valid update arises in two ways.

\myparagraph{Across-timesteps update} --
At the beginning of planning at time $t>0$, the certificate from time~$t-1$ is inherited and truncated by one step:
\begin{align*}
    \vtau_{\mathrm{cert}}^{A,t}(x_t)&=\{ \traj_{\mathrm{cert}}^{a_1,t}, \ldots, \traj_{\mathrm{cert}}^{a_N,t} \}, \\
    \traj_{\mathrm{cert}}^{a_i,t} &= \traj^{a_i,t-1}_{\mathrm{cert}}[1:] = [v_{\mathrm{cert}}^{a_i,t\mid t-1},\ldots,v_{\mathrm{cert}}^{a_i,M_t\mid t-1}].
\end{align*}
Because the executed movement is the first step of the previous certificate, the inherited trajectory set remains conflict-free and satisfies \cref{eq:budget-update-across-timesteps} with equality.

\begin{definition}[Suboptimal conflict-free tail] \label{def:suboptimal-cf-tail}
Given a conflict-free active prefix~$\vtau_{\mathrm{finite}}^{A,h_r\mid H_{\max}}$, a \emph{suboptimal conflict-free tail} is a conflict-free trajectory set returned by~$\pi_{\mathrm{bk}}$ from the terminal state of that prefix to the agents' goals.
\end{definition}

\myparagraph{Within-timestep update} --
Within a planning step, the closed-loop algorithm searches for a better certificate.
In the \gls{acr:accbs}-based instantiation developed later, a conflict-free active prefix is concatenated with a suboptimal conflict-free tail to form a certificate candidate~$\widetilde{\vtau}^{A}(x_t)$.
Whenever
\begin{equation*}
    \mathcal{C}(\widetilde{\vtau}^{A}(x_t)) < \mathcal{B}_t,
\end{equation*}
the incumbent certificate is replaced by
\begin{equation*}
    \vtau_{\mathrm{cert}}^{A,t}(x_t) \leftarrow \widetilde{\vtau}^{A}(x_t),
\qquad
\mathcal{B}_t \leftarrow \mathcal{C}(\widetilde{\vtau}^{A}(x_t)).
\end{equation*}
If no improving candidate is found within the allotted computation time, the incumbent certificate remains in force.

\begin{remark}[What the certificate filter guarantees]
The certificate filter ensures that every executed movement is the first step of a conflict-free full-horizon plan.
Its role is not to claim that every finite-horizon search decision is itself globally good, but to guarantee that the system always retains a feasible fallback and that accepted within-timestep updates strictly improve the certified cost upper bound.
\end{remark}

The following result is immediate.

\begin{lemma}[Strict decrease of the fleet budget and guaranteed completeness] \label{lem:strict-decrease}
Suppose a closed-loop \gls{acr:mapf} algorithm performs valid certificate updates at every timestep.
If at least one agent has not yet reached its goal at time~$t-1$, then~$\mathcal{B}_t < \mathcal{B}_{t-1}$.
Consequently, all agents reach their goals in finite time.
\end{lemma}
\begin{proof}
If some agent has not yet reached its goal at time~$t-1$, then~$\sum_{a_i\in A} p^{a_i}(x_{t-1}(a_i)) \ge 1.$
By \cref{eq:budget-update-across-timesteps},
\begin{equation*}
\mathcal{B}_t
\le
\mathcal{B}_{t-1} - \textstyle\sum_{a_i\in A} p^{a_i}(x_{t-1}(a_i))
\le
\mathcal{B}_{t-1}-1
<
\mathcal{B}_{t-1}.
\end{equation*}
Because the cost function in \cref{def:cost-function} is integer-valued and nonnegative, the fleet budget is a nonnegative integer.
Therefore, it can decrease strictly only finitely many times.
The process terminates when~$\mathcal{B}_t=0$, which is possible only when every agent is already at its goal.
\end{proof}

\section{Budget-Limited Reachability and Inheritable Factorization} \label{sec:factorization}
Certificates do more than certify feasibility.
Because the controller accepts only certificate-improving updates, future fleet behavior remains confined to trajectories whose total cost does not exceed the current fleet budget.
This observation yields a budget-limited notion of reachability and, from it, a factorization that remains valid across future timesteps.

\subsection{Slackness and budget-limited reachable regions}
Consider a one-shot \gls{acr:mapf} instance~$\cI$ at time~$t$ with system state~$x_t$ and current certificate~$(\vtau_{\mathrm{cert}}^{A,t}(x_t),\mathcal{B}_t)$.
We first isolate the part of the certificate budget that exceeds the sum of per-agent shortest path lower bounds.

\begin{definition}[Slackness] \label{def:slackness}
    At time~$t$, the \emph{slackness}~$\calS_t$ under the fleet budget $\cB_t$ is
    \begin{equation*}
        \calS_t = \calB_t - \textstyle\sum_{a_i} \gamma(x_t(a_i),a_i)
    \end{equation*}
    where~$\gamma(x_t(a_i),a_i)$ follows~\cref{def:cost-function}.
\end{definition}

\begin{remark}[Interpretation of Slackness] \label{rem:slackness-freedom}
The slackness is a fleet-level excess-cost budget.
It measures how much total detour or delay remains available beyond the sum of individual shortest-path costs from the current state.
\end{remark}

This slackness induces a space-time reachable region that any valid future certificate must respect.

\begin{definition}[Budget-limited reachable region] \label{def:budget-limit-reachable-region}
For an agent~$a_i\in A$ at current vertex $v_t^{a_i}$, given the fleet budget $\cB_t$ and slackness~$\calS_t$, the \emph{budget-limited reachable region} $\mathcal{R}^{\cB_t}(a_i) \subset V$ is the set of vertices $v$ the agent can occupy in the future without violating the fleet budget. It is formally defined as:
    \begin{equation*}
        \mathcal{R}^{\cB_t}(a_i) = \{v \in V \mid D(v,x_t,a_i) \leq \calS_t \},
    \end{equation*}
    where
    \begin{equation*}
        D(v,x_t,a_i)=d(x_t(a_i), v) + \gamma(v,a_i) - \gamma(x_t(a_i),a_i),
    \end{equation*}
    and~$d(u, v)$ denotes the shortest path distance from vertex $u$ to $v$ in $G$.
\end{definition}

The following result follows immediately.

\begin{lemma}[Budget-limited reachability] \label{lem:budget-limit-reachability}
Let~$v \notin \mathcal{R}^{\cB_t}(a_i)$.
Then any joint trajectory set~$\vtau^{A}(x_t)$ in which agent $a_i$ occupies $v$ at time $\tau$ must satisfy~$\mathcal{C}(\vtau^{A}(x_t)) > \mathcal{B}_t.$
Therefore, it cannot update the incumbent certificate trajectories with cost~$\calB_t$.
\end{lemma}

\begin{proof}
    Assume $v \notin \mathcal{R}^{\cB_t}(a_i)$. By \cref{def:budget-limit-reachable-region}, we have $D(v,x_t,a_i) > \calS_t$.
    For any trajectory set $\vtau^A(x_t)$ where $\traj^{a_i}$ occupies $v$, the cost for agent $a_i$ is at least the cost to reach $v$ plus the shortest path to its goal:
    $\calC(\traj^{a_i}) \geq d(x_t(a_i), v) + \gamma(v, a_i)$.
    For all other agents $a_j \neq a_i$, the minimum possible cost is their optimal distance to their respective goals: $\calC(\traj^{a_j}) \geq \gamma(x_t(a_j), a_j)$. Summing these bounds yields:
    \begin{equation*} \textstyle
        \calC(\vtau^A(x_t)) \geq d(x_t(a_i), v) + \gamma(v, a_i) + \sum_{a_j \neq a_i} \gamma(x_t(a_j), a_j).
    \end{equation*}
    Rearranging the inequality, one obtains:
    \begin{align*} \textstyle
        \calC(\vtau^A(x_t)) \geq \; &\big[ d(x_t(a_i), v) + \gamma(v, a_i) - \gamma(x_t(a_i), a_i) \big] \\ &+ \textstyle\sum_{a_k \in A} \gamma(x_t(a_k), a_k).
    \end{align*}
    The bracketed term is $D(v,x_t,a_i)$, and by \cref{def:slackness}, the sum equals $\calB_t - \calS_t$. Thus:
    \begin{equation*} \textstyle
        \calC(\vtau^A(x_t)) \geq D(v,x_t,a_i) + \calB_t - \calS_t > \calS_t + \calB_t - \calS_t = \calB_t.
    \end{equation*}
    This strict inequality proves~$\calC(\vtau^A(x_t)) > \calB_t$.
\end{proof}

\subsection{Factorization and inheritability}
Using the budget-limited reachable region, we can partition the fleet into independent groups, a process we refer to as \emph{budget-limited factorization}.

\begin{definition}[Budget-limited factorization] \label{def:budget-limited-factorization}
A \emph{budget-limited factorization} partitions the agent fleet~$A$ into disjoint groups~$A_1,\ldots, A_K$ where for any~$i\neq j$, $A_i \cap A_j = \emptyset$ and~$\bigcup A_k=A$.
The partition is based on the intersection of the budget-limited reachable regions (\cref{def:budget-limit-reachable-region}) of each agent, and, for any~$a_i \in A_{k_1}$,~$a_j \in A_{k_2}$,~$k_1\neq k_2$:
    \begin{equation*}
        \calR^{\calB_t}(a_i) \cap \calR^{\calB_t}(a_j) = \emptyset
    \end{equation*}
\end{definition}

\begin{remark}[Global vs. local factorization]
Unlike the factorization framework proposed in~\cite{li2025fico}, the budget-consumption factorization here is \emph{global}, meaning the factorized groups remain independent across all future timesteps.
This inheritability is established formally by~\cref{lem:shrinkage,cor:inheritability}.
\end{remark}

\begin{lemma}[Shrinkage of the budget-limited reachable region] \label{lem:shrinkage}
For any algorithm satisfying \cref{def:valid-update}, any agent $a_i\in A$, and any two timesteps $t<t'$:
\begin{equation*}
    \mathcal{R}^{\mathcal{B}_{t'}}(a_i) \subseteq \mathcal{R}^{\mathcal{B}_t}(a_i).
\end{equation*}
\end{lemma}

\begin{proof}
It suffices to prove the claim for $t'=t+1$; the general case then follows by induction.
By combining \cref{def:slackness} and~\cref{def:budget-limit-reachable-region}, we have:
\begin{equation} \label{eq:simplified-reach} \textstyle
        d(x_t(a_i),v) + \gamma(v,a_i) + \sum_{k \neq i} \gamma(x_t(a_k), a_k) \leq \calB_t.
\end{equation}
We first prove $\calR^{\calB_{t'}} \subseteq \calR^{\calB_t}$.
For any $v \in \calR^{\calB_{t+1}}(a_i)$, \cref{eq:simplified-reach} is satisfied at $t+1$.
Applying the triangle inequality, we get:
    \begin{align*}
        & d(x_t(a_i), v) + \gamma(v, a_i) + \textstyle \sum_{k \neq i} \gamma(x_t(a_k), a_k) \\
        \leq \; & [d(x_{t+1}(a_i), v) + \gamma(v, a_i) + \textstyle \sum_{k \neq i} \gamma(x_{t+1}(a_k), a_k)] \\ & + \textstyle\sum_{a_k \in A} p^{a_k}(x_t(a_k))
    \end{align*}
    Because $v \in \calR^{\calB_{t+1}}(a_i)$, the bracketed sum is bounded by $\calB_{t+1}$. The entire expression is bounded by $\calB_{t+1} + \sum_{a_k \in A} p^{a_k}(x_t(a_k))$, which exactly equals $\calB_t$ per~\cref{def:valid-update}.
    Thus, $v \in \calR^{\calB_t}(a_i)$, meaning $\calR^{\calB_{t+1}} \subseteq \calR^{\calB_t}$.

\end{proof}

\begin{corollary}[Inheritability of the budget-limited factorization] \label{cor:inheritability}
If two groups are independent under a budget-limited factorization at time $t$, then they remain independent at every future timestep $t'>t$.
\end{corollary}

\begin{proof}
Let $A_{k_1}$ and $A_{k_2}$ be two groups that are independent at time $t$. Then, for every $a_i\in A_{k_1}$ and $a_j\in A_{k_2}$, their budget-limited reachable regions do not intersect: $\calR^{\calB_t}(a_i) \cap \calR^{\calB_t}(a_j) = \emptyset$.
By \cref{lem:shrinkage}, for any $t' > t$, the reachable regions shrink, meaning~$\calR^{\calB_{t'}}(a_i) \subseteq \calR^{\calB_t}(a_i)$ and $\calR^{\calB_{t'}}(a_j) \subseteq \calR^{\calB_t}(a_j)$.
Since subsets of disjoint sets are themselves disjoint, $\calR^{\calB_{t'}}(a_i) \cap \calR^{\calB_{t'}}(a_j) = \emptyset$ necessarily holds.
Therefore, the groups remain independent at time $t'$, proving the factorization is globally inheritable.
\end{proof}

\section{\gls{acr:daccbs} Algorithm}

\makeatletter
\xpatchcmd{\algorithmic}{\itemsep\z@}{\itemsep=1pt}{}{}
\makeatother

\begin{algorithm}[t]
\resizebox{\columnwidth}{!}{
\begin{minipage}{\columnwidth}
    \small
    \caption{\gls{acr:daccbs} algorithm}
    \label{alg:daccbs}
    \begin{algorithmic}[1]
    \Require {Current instance $\calI$, current state $x_t$, set of groups with last factorization slackness $\calG^t=\{A_0, \ldots, A_K\}$, $\calS_{\mathrm{last}}=\{(A_0, \calS_{\mathrm{last}}^{A_0}),\ldots,(A_K, \calS_{\mathrm{last}}^{A_K})\}$, each group $A_k$'s previous certificate trajectories$\vtau_{\mathrm{cert}}^{A^k,t-1}$ and the fleet budget $\calB^{A_k}_{t-1}$, the backup controller $\pi_{\mathrm{bk}}$ nominal maximal horizon size $H_\max$, time limit $t_\max$, slackness drop threshold $\calT_{\calS}$}

    \ForAll{$A_k \in \calG^t$} \Comment{Parallel individual group planning}

        \LComment{\textbf{Step \circled{1}: Update certificate and budget}}
        \State $\vtau_{\mathrm{cert}}^{A_k,t}, \calB^{A_k}_t \gets \Call{UpdateCertBudget}{\calI, \vtau_{\mathrm{cert}}^{A_k,t-1}, \calB^{A_k}_{t-1}}$

        \LComment{\textbf{Step \circled{2}: Constraint tree initialization}}
        \State $h_{\mathrm{r}}=1$; $\mathrm{OPEN} \gets \mathsf{Empty Queue}$

        \State $N_{\mathrm{root}} \gets \Call{RootNodeGen}{\calI,x_t,H_{\max}}$
        \State $\mathrm{OPEN}.\textbf{insert}(N_{\mathrm{root}})$

        \LComment{\textbf{Step \circled{3}: Constraint tree construction}}

        \While{$\mathrm{OPEN} \neq \mathsf{EmptyQueue}$}
            \If{$\Call{TimeOut}{t_\max}$}
                \State \textbf{break}
            \EndIf
            \State $n \gets \mathrm{OPEN}.\mathbf{dequeue}()$
            \State $\mathrm{Conflict},a_i,a_j \gets \Call{FindConf}{n.\vtau_{\mathrm{finite}}^{A,H_\max},h_{\mathrm{r}}}$
            \If{$\mathrm{Conflict}=\emptyset$}
                \LComment{\textbf{\circled{4}: Increment $h_r$ and update certificates}}
                \State $\vtau^{A_k}_{\mathrm{tail}} \gets \Call{BKRollout}{n.\vtau_{\mathrm{finite}}^{A,H_\max},h_{\mathrm{r}},\pi_{\mathrm{bk}}}$
                \If{$\calC(\Call{Concat}{n.\vtau_{\mathrm{finite}}^{A,H_\max}[0:h_r],\vtau^{A_k}_{\mathrm{tail}}}) < \calB^{A_k}_t$}
                    \State $\vtau_{\mathrm{cert}}^{A_k,t} \gets \Call{Concat}{n.\vtau_{\mathrm{finite}}^{A,H_\max}[0:h_r],\vtau^{A_k}_{\mathrm{tail}}}$
                    \State $\calB^{A_k}_t \gets \calC(\vtau_{\mathrm{cert}}^{A_k,t})$
                \EndIf
                \While{$\mathrm{Conflict}=\emptyset \ \textbf{and} \ h_{\mathrm{r}} < H_\max$}
                    \State $h_{\mathrm{r}} \gets h_{\mathrm{r}}+1$
                    \State $\mathrm{Conflict},a_i,a_j \gets  \Call{FindConf}{n.\vtau_{\mathrm{finite}}^{A,H_\max},h_{\mathrm{r}}}$
                \EndWhile
                \If{$h_{\mathrm{r}}=H_\max$}
                    \State \textbf{break}
                \EndIf
            \EndIf
            \State $n_{\mathrm{new}}^{a_i},n_{\mathrm{new}}^{a_j} \gets \Call{ChildNodeGen}{\calI,x_t,\mathrm{Conf}}$
            \State $\mathrm{OPEN}.\mathbf{insert}(n_{\mathrm{new}}^{a_i},n_{\mathrm{new}}^{a_j})$
        \EndWhile
        \LComment{\textbf{\circled{5}: Slackness based further factorization}}
        \State $\calS_t \gets \Call{ComputeSlackness}{\calI, A_k, x_t}$
        \If{$\calS_{\mathrm{last}}(A_k) - \calS_t \geq \calT_{\calS}$}
            \State $\calG_k, \calS_{\mathrm{last}} \gets \Call{BLFact}{\calI, A_k, \calS_t}$
        \EndIf
    \EndFor

    \State $\calG \gets \Call{Concat}{\calG_0, \ldots, \calG_K}$
    \State \Return $ \Call{ExtractFirstStepMovement}{\vtau_{\mathrm{cert}}^{A_0},\ldots,\vtau_{\mathrm{cert}}^{A_K}}$
    \end{algorithmic}
\end{minipage}
}
\end{algorithm}

Combining the certificates implementation (\cref{sec:certificates}) and the budget-limited factorization mechanism (\cref{sec:factorization}) with the \gls{acr:accbs} algorithm~\cite{li2026adaptive} yields the \gls{acr:daccbs} algorithm (\cref{alg:daccbs}).
In this section, we formally present the algorithm and its theoretical properties.

\subsection{Details of the \gls{acr:daccbs} algorithm}
At a single control step $t$ of the \gls{acr:mapf} system, similar to the \gls{acr:accbs} algorithm, \gls{acr:daccbs} builds a constraint tree over~$H_{\max}$, but enforces conflict-freedom only within the active prefix of length~$h_r$.
The running horizon starts at~$h_r=1$  and grows as long as time allows, and conflicts can be resolved.
Besides these processes in \gls{acr:accbs}, \gls{acr:daccbs} consistently maintains the certificate trajectories and the corresponding fleet budget across timesteps and updates the certificates as per~\cref{sec:updating-certificates} whenever a conflict-free active prefix is found.
At any time, \gls{acr:daccbs} can return the first-step movement corresponding to the current certificate trajectories.

\myparagraph{\circled{1} Initialization and inheritance} --
\gls{acr:daccbs} inherits the factorization result from the previous time step, allowing all agent groups to be processed in parallel.
For each agent group~$A_k$, a set of certificate trajectories~$\vtau^{A_k,t-1}_{\mathrm{cert}}$ and its associated cost, the fleet budget~$\calB^{A_k}_{t-1}$, are inherited and updated according to the across timesteps update in~\cref{sec:updating-certificates}:
\begin{align}
    \textstyle &\vtau^{A_k,t}_{\mathrm{cert}} = \{[v_{\mathrm{cert}}^{a_i,t\mid t-1},\ldots,v_{\mathrm{cert}}^{a_i,M_t\mid t-1}]\mid a_i\in A_k\} \\
    \textstyle &\calB^{A_k}_t = \textstyle \calB^{A_k}_{t-1} - \sum_{a_i}p^{a_i}(v_{t-1}^{a_i}) \label{eq:budget-update}
\end{align}

In the subsequent steps, \gls{acr:daccbs} aims to find an improved set of certificate trajectories with a lower cost, reducing the budget accordingly.

\myparagraph{\circled{2}, \circled{3} Constraint tree initialization and construction} --
At the beginning of each time step, the running horizon~$h_r$ is initialized to 1.
Let~$\mathcal{T}_1$ denote the corresponding constraint tree for~$h_r=1$.
$\mathcal{T}_1$ is initialized with a root node containing an empty constraint set and a set of individually optimal~$H_{\max}$-step trajectories for all agents.
The main loop performs a best-first search over the constraint tree based on node costs, dynamically extending the running horizon whenever a node with a conflict-free active prefix is found.
In each iteration, the node~$n$ with the lowest cost is extracted, and its active prefix~$\vtau^{A^t,h_r\mid H_\max}_{\mathrm{finite}}(n)$ is checked for conflicts.
If a conflict is detected, two child nodes are generated by adding a new constraint to the constraint set for each of the involved agents.
Their trajectories are then replanned subject to the updated constraint sets, mirroring the standard \gls{acr:cbs} algorithm (\cref{sec:preliminary-cbs}).
Alternatively, if no conflicts exist within the active prefix of~$n$, the algorithm moves to the following step.

\myparagraph{\circled{4} Running horizon increment and certificate update} --
When a node with a conflict-free active prefix is found, the algorithm attempts to extend $h_r$. We can show that the node cost remains invariant under this extension~\cite[Lemma~III.1]{li2026adaptive}, which directly leads to the full reusability of the constraint tree~\cite[Proposition~III.1]{li2026adaptive} which eliminates duplicated computation.

In addition to extending the running horizon, finding a node with a conflict-free active prefix triggers a within-timestep update (\cref{sec:updating-certificates}) on the certificate trajectories and the fleet budget, where the candidate certificate trajectories are constructed by concatenating the active prefix with the suboptimal conflict-free tail (\cref{def:suboptimal-cf-tail}). Incumbent certificates are updated once better certificates are found.
Within a preset time limit, the main loop repeatedly attempts to update the certificates and extend the running horizon.

\myparagraph{\circled{5} Budget-limited factorization } --
The budget-limited factorization is attempted once the certificate update is concluded.
As noted in~\cref{rem:slackness-freedom} and~\cref{lem:shrinkage}, the fleet's slackness dictates the size of the budget-limited reachable region, and a decrease in slackness shrinks these regions and potentially leads to a finer factorization.
To balance factorization sensitivity and computational burden, \gls{acr:daccbs} introduces a slackness drop threshold, denoted as~$\calT_{\calS}$.
If the current slackness drops by more than~$\calT_{\calS}$ relative to the value at the previous factorization step, \gls{acr:daccbs} recomputes the budget-limited reachable regions (\cref{def:budget-limit-reachable-region}), enabling a formal factorization via disjoint-set data structures (DSU)~\cite{cormen2022introduction}.
Thanks to the inheritability property (\cref{cor:inheritability}), these newly formed independent groups are pooled and seamlessly carried over to future time steps.

\subsection{Theoretical analysis} \label{sec:theoretical-analysis}
To evaluate solution quality, we adopt the standard \gls{acr:soc} metric, which can be interpreted as total energy consumption~\cite{surynek2016empirical}.

\begin{definition}[\gls{acr:soc}]
Let~$\vtau^{A}_{\mathrm{sol}} = \{\traj^{a_1}_{\mathrm{sol}},\ldots,\traj^{a_N}_{\mathrm{sol}}\}$ be solution trajectories of the one-shot \gls{acr:mapf} problem, where~$\traj_{\mathrm{sol}}^{a_i}=[v^{a_i}_0, \ldots, v^{a_i}_M]$ is the trajectory of agent~$a_i$.
The \gls{acr:soc} is
    \[ \textstyle
        \mathrm{SOC}(\vtau^{A}_{\mathrm{sol}})
        = \sum_{a_i} \mathrm{SOC}(\traj_{\mathrm{sol}}^{a_i})
        = \sum_{a_i} \sum_t p^{a_i}(v^{a_i}_t),
    \]
    where~$p^{a_i}$ is the running cost~\cref{def:cost-function}.
\end{definition}

We consider an idealized case where (i)~$H_{\max}$ upper-bounds the makespan of some \gls{acr:soc}-optimal solution, and (ii) the time for running horizon extension~$t_{\max}$ is infinite.
Then $h_r$ increases monotonically until it reaches~$H_{\max}$, meaning the algorithm never terminates prematurely.

In this case, the following result can be established.

\begin{lemma}[\gls{acr:soc} optimality of \gls{acr:daccbs}] \label{lem:optimality}
    The \gls{acr:daccbs} algorithm terminates in finite time and returns a conflict-free joint trajectory of length~$H_{\max}$ that is globally optimal under the stated choice of~$H_{\max}$ and infinite computation time.
\end{lemma}

\begin{proof}
    Given $t_{\max} \to \infty$, the algorithm monotonically extends the running horizon without premature termination until~$h_r = H_{\max}$.
    By~\cite[Proposition~III.1]{li2026adaptive}, the constraint tree is strictly preserved across extensions.
    By~\cite[Lemma~III.1]{li2026adaptive}, the node cost~$J_{h_r}(n)$ is invariant and strictly equals the full trajectory cost~$J_{H_{\max}}(n)$.

    At~$h_r = H_{\max}$, the search space and constraint resolution mechanisms are identical to standard \gls{acr:cbs} over the entire spatial-temporal domain. Since standard \gls{acr:cbs} is complete and optimal for the \gls{acr:soc} objective, and $H_{\max}$ upper-bounds the optimal makespan, the first conflict-free joint trajectory extracted guarantees global optimality.
\end{proof}

Compared with~\gls{acr:accbs}~\cite{li2026adaptive}, the certificate mechanism enables \gls{acr:daccbs} to directly link the closed-loop solution quality with the allocated computational budget per step, as discussed in the following remark.

\begin{remark}[Monotone improvement of the \gls{acr:soc} upper bound] \label{rem:improving-quality}
    A longer $t_{\max}$ corresponds to a longer running horizon that the algorithm can extend, resulting in more attempts to update the certificate.
    Unlike \gls{acr:accbs} where $t_{\max}$ is not linked to \gls{acr:soc} performance~\cite[Remark~III.4]{li2026adaptive}, in \gls{acr:daccbs}, the fleet budget upper-bounds the \gls{acr:soc} performance, and decreases monotonically as $t_{\max}$ increases.
\end{remark}

\section{Experiments}  \label{sec:experiments}
We evaluate three questions.
First, does the certificate mechanism improve the closed-loop solution quality of \gls{acr:daccbs} relative to \gls{acr:accbs}?
Second, does budget-limited factorization expose practically useful compositional structure? Third, how sensitive is the framework to the choice of backup controller?
All experiments were run on a 2023 MacBook Pro with a 12-core CPU and 36~GB of RAM.
Planning across groups is parallelized over 12 threads.
Unless otherwise stated, LaCAM~\cite{okumura2023lacam} is used as the backup controller for certificate generation, and all maps are taken from the \gls{acr:mapf} benchmark~\cite{stern2019mapf}.

\myparagraph{\gls{acr:soc} performance comparison} --
\Cref{fig:one-shot-exp} compares \gls{acr:daccbs}, \gls{acr:accbs}, and LaCAM under different per-step planning budgets.
The vertical axis reports the average \gls{acr:soc} increment, so lower values indicate better performance.
Across all tested maps, \gls{acr:daccbs} is consistently competitive with or better than \gls{acr:accbs}, and the advantage becomes more pronounced as the occupancy rate increases.
This effect is especially visible on the random maps, where \gls{acr:accbs} degrades rapidly in dense regimes while \gls{acr:daccbs} remains substantially lower.

A second trend is that \gls{acr:daccbs} improves more predictably as additional planning time is provided.
Its curves decrease steadily as the per-step computational budget increases, whereas \gls{acr:accbs} does not exhibit the same monotone behavior.
This is consistent with the role of certificates: additional computation is used to tighten a feasible incumbent certificate rather than to commit to a potentially unstable finite-horizon plan.
The theory guarantees monotonic improvement of the certified upper bound (\cref{rem:improving-quality}); empirically, the results indicate that this also translates into more stable realized closed-loop performance.

\begin{figure}[tb]
    \centering
    \subfloat[\href{https://movingai.com/benchmarks/mapf/empty-8-8.png}{\texttt{Empty Map 1}}]{\includegraphics[width=0.5\linewidth]{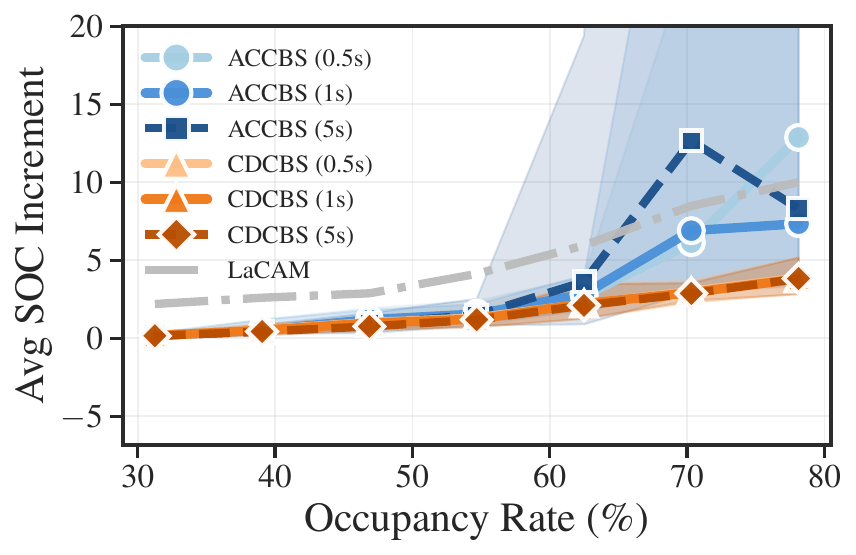}} \hfill
    \subfloat[\href{https://movingai.com/benchmarks/mapf/random-32-32-10.png}{\texttt{Random Map 1}}]{\includegraphics[width=0.5\linewidth]{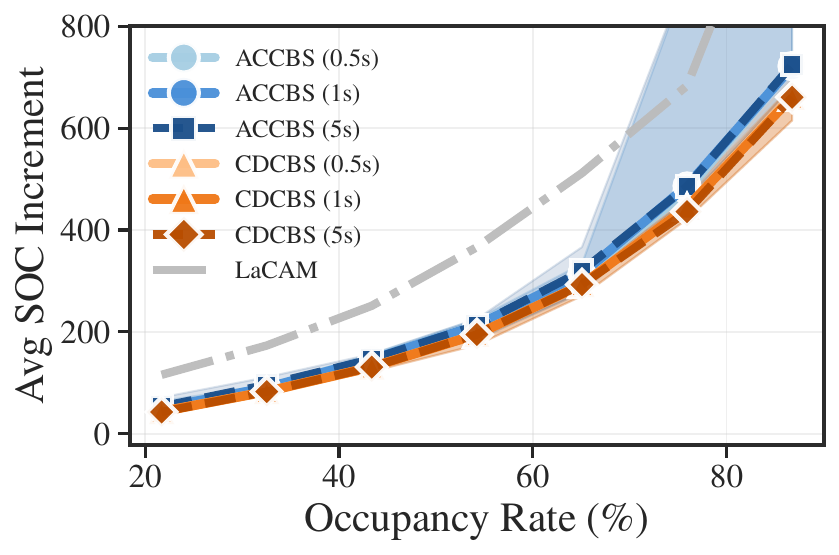}}
    \\
    \subfloat[\href{https://movingai.com/benchmarks/mapf/empty-48-48.png}{\texttt{Empty Map 2}}]{\includegraphics[width=0.5\linewidth]{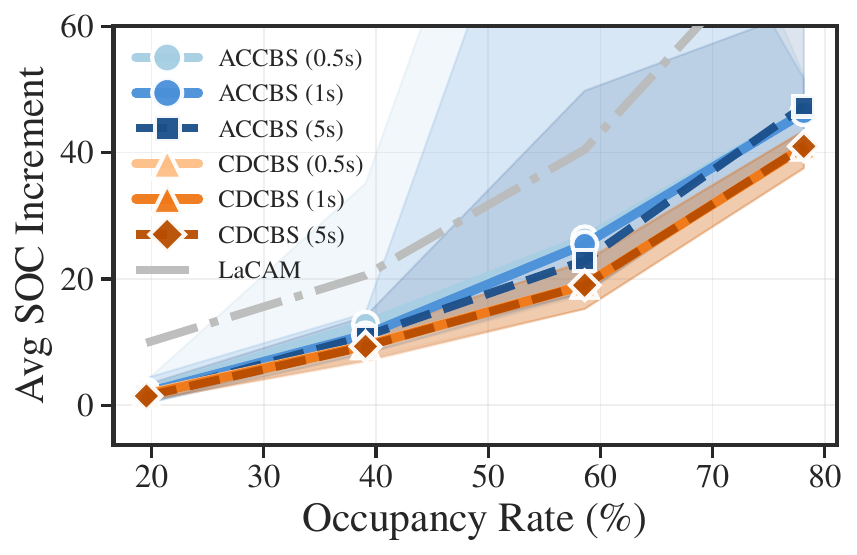}} \hfill
    \subfloat[\href{https://movingai.com/benchmarks/mapf/random-64-64-10.png}{\texttt{Random Map 2}}]{\includegraphics[width=0.5\linewidth]{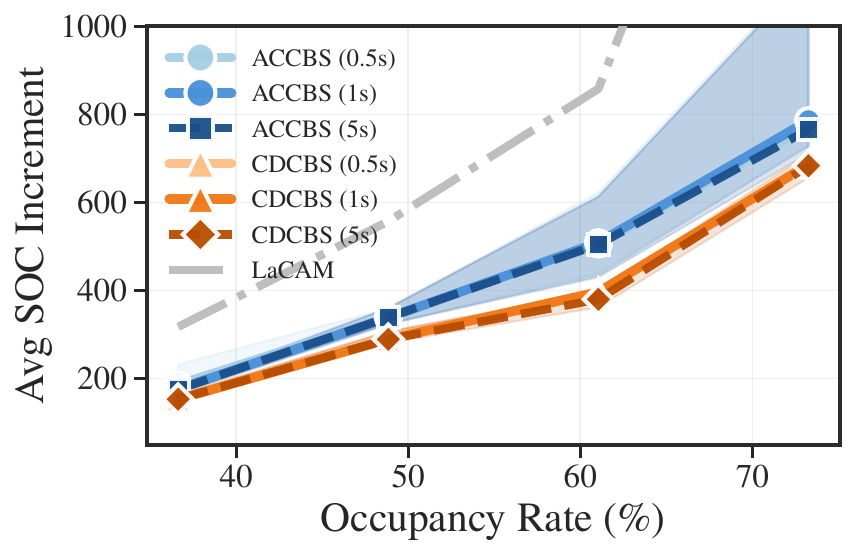}}

    \caption{With the certificates, the \gls{acr:daccbs} can achieve significantly better performance than~\gls{acr:accbs}, which does not have the certificate, especially in denser cases. }
    \label{fig:one-shot-exp}
\end{figure}

\begin{table}[tb]
    \centering
    \caption{Factorization results at $t=\lceil \text{makespan}/2 \rceil$}
    \label{tab:factorization}

    \begin{minipage}[t]{0.47\linewidth}
        \centering
        \parbox[c][1.2em][c]{\textwidth}{\centering\textbf{\href{https://movingai.com/benchmarks/mapf/empty-48-48.png}{\texttt{Empty Map 2}}}}
        \vspace{0.5em}
        \begin{tabular}{@{}ccc@{}}
            \toprule
            \textbf{\ $N$} & $K$ & \textbf{${\max|A_{\mathrm{k}}|}/|A|$} \\ [0.3em]
            \textbf{} & \textbf{(Groups} & \textbf{(Computation} \\
            \textbf{} & \textbf{count)} & \textbf{reduction)} \\
            \midrule
            10 & 6 & 0.20  \\
            30 & 8 & 0.47  \\
            50 & 6 & 0.44 \\
            70 & 24 & 0.36 \\
            200 & 1 & 1.00 \\
            \bottomrule
        \end{tabular}
    \end{minipage}
    \begin{minipage}[t]{0.47\linewidth}
        \centering
        \parbox[c][1.2em][c]{\textwidth}{\centering\textbf{\href{https://movingai.com/benchmarks/mapf/random-32-32-20.png}{\texttt{Random Map 2}}}}
        \vspace{0.5em}
        \begin{tabular}{@{}ccc@{}}
            \toprule
            \textbf{\ $N$} & $K$ & \textbf{${\max|A_{\mathrm{k}}|}/|A|$} \\ [0.3em]
            \textbf{} & \textbf{(Groups} & \textbf{(Computation} \\
            \textbf{} & \textbf{count)} & \textbf{reduction)} \\
            \midrule
            10 & 9 & 0.20  \\
            30 & 26 & 0.10  \\
            50 & 27 & 0.36 \\
            70 & 21 & 0.35 \\
            200 & 1 & 1.00 \\
            \bottomrule
        \end{tabular}
    \end{minipage}
\end{table}

\myparagraph{Budget-limited factorization} --
\gls{acr:daccbs} uses budget-limited factorization to partition the fleet into independent groups that can be inherited across future timesteps.
\Cref{tab:factorization} summarizes the resulting decompositions.
The reported number of groups and the reduction in average group size show that the fleet can often be split into substantially smaller subproblems, especially under tighter slackness budgets.
This trend is consistent with the theory.
Smaller slackness yields smaller budget-limited reachable regions, making disjointness easier to certify and factorization more effective.
Although \Cref{tab:factorization} reports structural decomposition rather than wall-clock speedups, these reductions indicate meaningful potential gains when per-group planning is combined with parallel computation.

\begin{figure}[t]
    \centering
    \subfloat[\href{https://movingai.com/benchmarks/mapf/empty-16-16.png}{\texttt{Empty Map 2}}]{\includegraphics[width=0.5\linewidth]{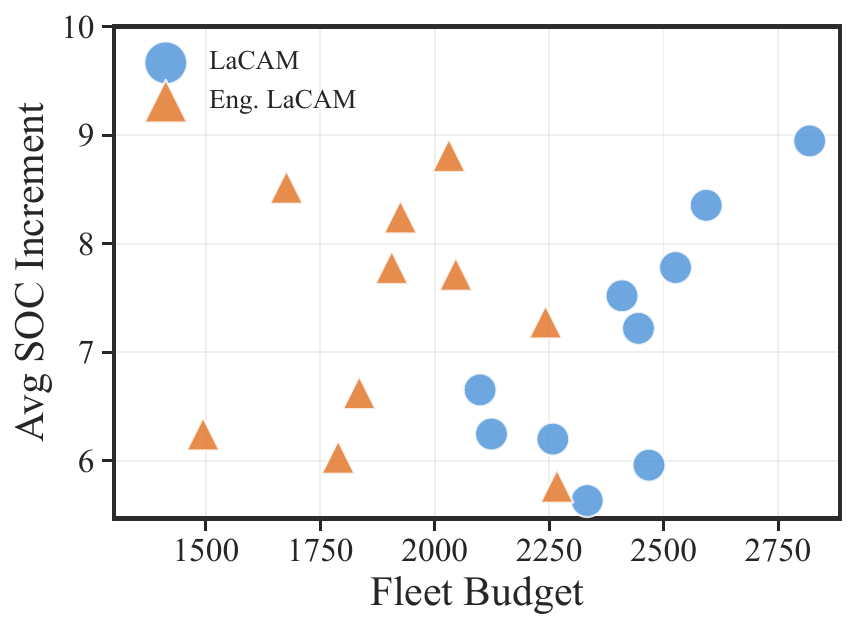}} \hfill
    \subfloat[\href{https://movingai.com/benchmarks/mapf/random-32-32-20.png}{\texttt{Random Map 2}}]{\includegraphics[width=0.5\linewidth]{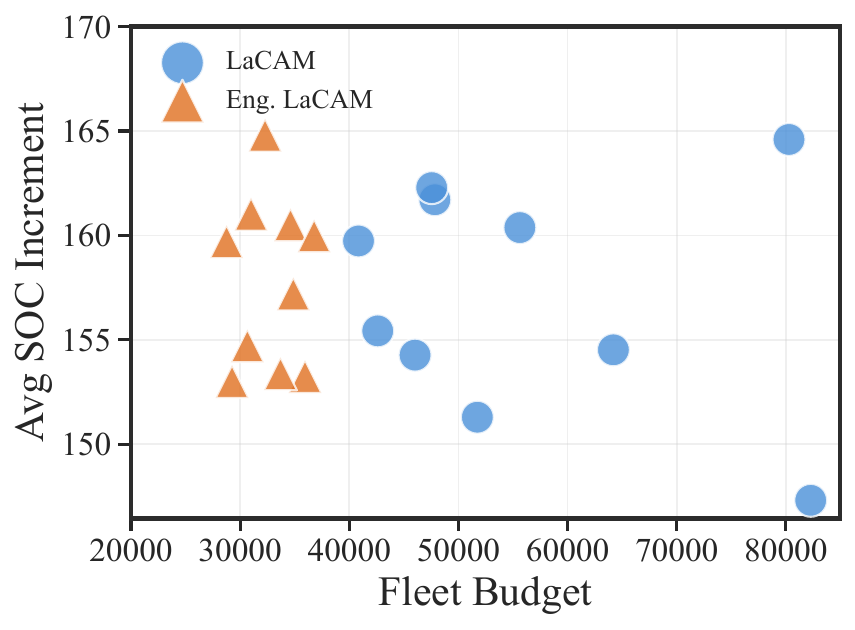}}

    \caption{Across planning budgets, certificate-guided \gls{acr:daccbs} achieves lower average \gls{acr:soc} increment than \gls{acr:accbs}, with advantages becoming most pronounced in denser instances.}
    \label{fig:backup-controller}
\end{figure}

\myparagraph{The choice of backup controller} --
The backup controller affects the quality of the certificates and therefore the tightness of the fleet budget.
\Cref{fig:backup-controller} compares vanilla LaCAM~\cite{okumura2023lacam} and engineered LaCAM~\cite{okumura2023engineering} as backup controllers.
The engineered variant typically produces tighter certificates, as reflected by lower fleet budgets, but the final closed-loop \gls{acr:soc} increments remain close to those obtained with vanilla LaCAM.
This highlights an important trade-off: improving certificate quality does not automatically translate into proportionally better end-to-end performance.
In practice, the backup controller should therefore be chosen by balancing certificate quality against the additional computation required to produce it.

\section{Conclusion and Future Work}

This paper introduced a certificate-based framework for closed-loop \gls{acr:mapf} and instantiated it on \gls{acr:accbs} to obtain \gls{acr:daccbs}.
The key idea is to maintain conflict-free certificate trajectories together with their fleet budget as an incumbent plan at every timestep, and to accept closed-loop updates only when they improve this certificate.
This restores a global notion of progress to finite-horizon closed-loop planning: monotone budget decrease yields completeness, while the same budget induces budget-limited reachable regions that support inheritable factorization.

The resulting algorithm combines the reactivity of closed-loop planning with two properties that are difficult to retain under finite-horizon approximations: a feasible fallback plan and exploitable compositional structure.
Empirically, \gls{acr:daccbs} yields more stable solution quality than \gls{acr:accbs}, particularly in dense instances, and the proposed factorization exposes substantial structure that can be leveraged for parallel planning.
The backup-controller exploration further shows that tighter certificates do not automatically imply proportionally better realized performance, which clarifies an important design trade-off in the framework.

Several directions follow naturally from this work.
One is to apply the certificate framework to closed-loop \gls{acr:mapf} algorithms beyond \gls{acr:accbs}.
Another is adaptive backup-controller selection, where certificate quality and certificate-generation cost are balanced online.
A third is a direct runtime study of inheritable factorization under different parallel budgets.
More broadly, certificates may provide a useful interface between learning-based decision-making and certified closed-loop multi-agent planning.

{
  \bibliographystyle{IEEEtran}
  \bibliography{references}
}

\end{document}